\documentclass[sigconf]{acmart}

\usepackage{amsthm}

\newtheorem{theorem}{Theorem}[section]

\newtheorem{lemma}[theorem]{Lemma}

\AtBeginDocument{%
  \providecommand\BibTeX{{%
    \normalfont B\kern-0.5em{\scshape i\kern-0.25em b}\kern-0.8em\TeX}}}

\copyrightyear{2021} 
\acmYear{2021} 
\setcopyright{acmlicensed}\acmConference[DocEng '21]{ACM Symposium on Document Engineering 2021}{August 24--27, 2021}{Limerick, Ireland}
\acmBooktitle{ACM Symposium on Document Engineering 2021 (DocEng '21), August 24--27, 2021, Limerick, Ireland}
\acmPrice{15.00}
\acmDOI{10.1145/3469096.3474937}
\acmISBN{978-1-4503-8596-1/21/08}

\begin{document}

\title{Efficient Sparse Spherical k-Means for Document Clustering}

\author{Johannes Knittel}
\email{johannes.knittel@vis.uni-stuttgart.de}
\orcid{0000-0002-4889-5232}
\affiliation{%
  \institution{University of Stuttgart}
  \city{Stuttgart}
  \country{Germany}
}

\author{Steffen Koch}
\email{steffen.koch@vis.uni-stuttgart.de}
\orcid{0000-0002-8123-8330}
\affiliation{%
  \institution{University of Stuttgart}
  \city{Stuttgart}
  \country{Germany}
}

\author{Thomas Ertl}
\email{thomas.ertl@vis.uni-stuttgart.de}
\orcid{0000-0003-4019-2505}
\affiliation{%
  \institution{University of Stuttgart}
  \city{Stuttgart}
  \country{Germany}
}

\renewcommand{\shortauthors}{Knittel et al.}

\begin{abstract}
Spherical $k$-Means is frequently used to cluster document collections because it performs reasonably well in many settings and is computationally efficient.
However, the time complexity increases linearly with the number of clusters $k$, which limits the suitability of the algorithm for larger values of $k$ depending on the size of the collection.
Optimizations targeted at the Euclidean $k$-Means algorithm largely do not apply because the cosine distance is not a metric.
We therefore propose an efficient indexing structure to improve the scalability of Spherical $k$-Means with respect to $k$.
Our approach exploits the sparsity of the input vectors and the convergence behavior of $k$-Means to reduce the number of comparisons on each iteration significantly.
\end{abstract}

\ccsdesc[500]{Information systems~Information extraction}
\ccsdesc[500]{Information systems~Summarization}

\keywords{document clustering, k-Means, large-scale analysis}


\maketitle

\section{Introduction}

Clustering algorithms facilitate the analysis of large data sets, particularly if they comprise unstructured data such as textual documents.
Spherical $k$-Means~\cite{Dhillon2001} has been proposed for the clustering of documents that have high-dimensional ($\gg 1000$) but very sparse vector representations.
It is based on $k$-Means~\cite{Lloyd1982} and replaces the Euclidean with the cosine distance to improve the performance on documents.
Benchmarks have shown that Spherical $k$-Means performs competitively on a variety of textual data sets compared to more advanced approaches and is computationally efficient if $k$ is sufficiently small~\cite{lelu:hal-03053176}.

However, on each iteration, we have to find the closest cluster centroid for every data item, which results in $\mathcal{O}(kN)$ comparisons where we have to compute the distance between two vectors.
The linear dependency of $k$ on the time complexity can lead to prohibitively large running times on larger data sets with $k \gg 10$, which limits the utility of Spherical $k$-Means for fine-grained analyses on big document collections.

Several approaches have been developed to increase the computational efficiency of $k$-Means.
Elkan~\cite{Elkan2003} applied the triangle inequality to reduce the number of distance calculations that have to be performed, but the triangle inequality does not hold for the cosine distance function.
Data structures for efficient nearest neighbor searches that are based on, for instance, k-d trees~\cite{Bentley1975}, coordinate-pruning~\cite{Teflioudi2016}, or product quantization codes~\cite{Johnson2019}, generally assume moderately-sized, dense input vectors, or are based on the triangle inequality.

Numerous online versions of $k$-Means have been proposed, including greedily updating the cluster centroids on each incoming element~\cite{Chakrabarti2006}, applying $k$-Means on separate batches and using the resulting centroids as input for the global clustering~\cite{Ailon2009}, or performing clustering only on a sampled subset of the data~\cite{Ackermann2010,Braverman2011}.
While computationally very efficient, these approaches only approximate the $k$-Means objective.

We propose an accelerated version of Spherical $k$-Means that can efficiently cluster large document collections even if $k \gg 10$.
We introduce an indexing structure that leverages the sparsity of the input vectors for an efficient (but non-approximated) retrieval of cluster centroids that maximize the cosine similarity.
We also exploit the observation that the number of changing cluster centroids typically decreases after several iterations.
Both strategies significantly reduce the average number of pairwise distance calculations that have to be performed on each iteration.

\section{Spherical k-Means}
\label{sec:sphericalKMeans}

The Euclidean $k$-Means algorithm~\cite{Lloyd1982} aims to find in an iterative way a partition of the data set that minimizes its clustering objective.
That is, every data item $\mathbf{x}_i$ gets associated with one of the $k$ clusters such that the sum of the squared differences between data items and their assigned centroids (i.e., mean of all associated items in the cluster) is minimal.
The algorithm initially assigns every item to a cluster (e.g., with the $k$-Means++ strategy).
Afterward, it optimizes the objective iteratively with a loop that comprises two steps.
First, the cluster centroids $c_j$ are recalculated based on the current assignments.
Then, the assignments are updated, that is, for every data item $\mathbf{x}_i$ we find the cluster $j$ such that the squared distance $\|\mathbf{x}_i - \mathbf{c}_j \|_2^2$ is minimized.
The loop stops if the assignments do not change anymore or if a specific criterion is met, e.g., the maximum difference between subsequent cluster centroids is sufficiently small.

In the case of document clustering, input vectors are often high-dimensional but sparse.
For instance, every term in the corpus may be assigned to an index position (the vocabulary) and then each document may be represented by a vector in which those entries are set to 1 (or to the corresponding TF-IDF weight~\cite{Salton1988}) that represent the present terms in the document.
For such document clustering use cases, Spherical $k$-Means~\cite{Dhillon2001} was proposed that applies the cosine similarity instead of the Euclidean distance.
Every input vector is supposed to have unit length, and the centroids are calculated as the sum of the associated items, divided by the length to obtain unit vectors as well.
The data items are then assigned to the cluster with the highest cosine similarity (which is equivalent to the dot product of the unit-length vectors).
Let $p,k$ be the number of non-zero vector entries of vectors $\mathbf{a}$ and $\mathbf{b}$, respectively, then the complexity of calculating the dot product $\mathbf{a}^\top \mathbf{b}$ is in $\mathcal{O}\left(\mathrm{min}(p, k)\right)$, given a map-like data structure of the vectors.
Thus, the sparser the input data, the faster the clustering.

\section{Method}

In this section, we propose two complementing strategies to accelerate the Spherical $k$-Means algorithm on sparse document representations.
The vast majority of the running time is spent on calculating the cosine similarity between an input vector and a cluster centroid.
Both strategies aim to reduce the average number of these computations that have to be performed.

\subsection{Non-Changing Clusters}
\label{sec:NCC}

Over the course of the iterations, the number of affected data items typically decreases, that means, less and less data items change their cluster association during the assignment step.
As a result, there is an increasing number of cluster centroids that stay the same in later iterations.

After recalculating the cluster centroids, we determine the set of centroids $C_u$ that have not changed compared to the previous iteration (within a certain tolerance $\epsilon$ to accommodate for rounding errors).
During the assignment step, we then check whether the current data item $\mathbf{x}_i$ was previously associated with a cluster in $C_u$.
If this is the case, we only need to calculate the similarity of $\mathbf{x}_i$ to centroids $c_j \not\in C_u$ that have actually changed and compare whether we get a higher similarity than with our previous association $a_i$, because we already know from the past iteration that $a_i$ relates to the most similar centroid among $C_u$.

\subsection{Dot Product Indexing Structure}

In the assignment step, we search for the centroid that maximizes the dot product with the current item.
For sparse input vectors, it could very well be the case that the non-zero entries of a centroid do not overlap with the non-zero entries of the current item, leading to a dot product of $0$.
To ignore such centroids, we could build an inverse index at the beginning of this step that maps an index to all centroids that have a non-zero value at that position.
Given an item, we can then enumerate through all of its non-zero values to retrieve the union of the centroids that share at least one non-zero entry.
The dot product with any remaining centroid is zero.
This can lead to a measurable speed-up if both the average number of non-zero input vector entries is sufficiently small and the centroids are distinct enough.
However, even in sparse settings we cannot generally assume that this is the case.

To improve our indexing structure we can exploit the fact that the input and centroid vectors have length one, and that the dot product with the (possibly updated) centroid based on the previous assignment will most likely be greater than zero.

\begin{lemma}
\label{lem:squaredSum}
Given two unit-length vectors $\mathbf{c} = [ c_1, ..., c_n ]^\top, \mathbf{x} = [ x_1, ..., x_n ]^\top \in \mathbb{R}^n$ and let $S = \lbrace a_1, ..., a_m \rbrace, a_i \in \mathbb{N}$ be the set of indexes that correspond to a non-zero value in $\mathbf{x}$. That is, $x_i \neq 0$ if and only if $i \in S$.
If $\mathbf{c} \cdot \mathbf{x} \geq \lambda$ then it holds that $\sum_{i=1}^{m}  c_{a_i}^2 \geq \lambda^2$.
\end{lemma}

\begin{proof}
Let $\mathbf{c} \cdot \mathbf{x} \geq \lambda, \hat{\mathbf{c}} = [ c_{a_1}, ..., c_{a_m} ]^\top, \hat{\mathbf{x}} = [ x_{a_1}, ..., x_{a_m} ]^\top$.
It follows that $\mathbf{c} \cdot \mathbf{x} = \hat{\mathbf{c}} \cdot \hat{\mathbf{x}}$, because $\hat{\mathbf{x}}$ comprises all non-zero entries of $\mathbf{x}$.
We can rewrite the dot product as follows: $\hat{\mathbf{c}} \cdot \hat{\mathbf{x}} = \Vert \hat{\mathbf{c}} \Vert \Vert \hat{\mathbf{x}} \Vert \cos \hat{\theta}$ where $\hat{\theta}$ is the angle between both vectors.
It follows that $\Vert \hat{\mathbf{c}} \Vert \cos \hat{\theta} \geq \lambda$, because $\Vert \hat{\mathbf{x}} \Vert = 1$.
It holds that $\cos \hat{\theta} \leq 1$.
Thus, it follows that $ \sqrt{\sum_{i=1}^{m}  c_{a_i}^2} = \Vert \hat{\mathbf{c}} \Vert \geq \lambda$.
\end{proof}

Applied to an input vector $\mathbf{x}$ and a centroid $\mathbf{c}$, it means that the cosine similarity can only be $\lambda$ or higher if the sum of the squared \emph{centroid} values of the overlapping non-zero entries equates to at least $\lambda^2$.
Based on this, we can build an indexing structure for a given minimum dot product of $\lambda$.

Given a sparse centroid $\mathbf{c}_i$ that we want to add to the structure, we first sort the index-value pairs of the present entries in $\mathbf{c}_i$ in descending order of their value.
For each pair, we then perform the following steps:

\begin{enumerate}
\item We add the index with the centroid ID $i$ to the general index map $G$. This corresponds to the basic indexing structure that we outlined at the beginning of this section. That is, for a given index, we can then retrieve a list of centroids that have a non-zero value at the corresponding position.
\item If the value is greater than or equal to $\lambda$, we immediately add the index with a \emph{minimum overlap count} of 1 and the ID to the index map $P$. If a query vector has a non-zero value at that position, it could already be enough to lead to a dot product $\geq \lambda$, hence the overlap count of 1.
\item If the value is lower than $\lambda$, a query vector cannot reach the required minimum dot product if it only shares a non-zero entry with $\mathbf{c}_i$ at that position. We iterate through the next pairs and sum up the squared values (including the current pair), until we reach our threshold of $\lambda^2$. The number of affected pairs is then our minimum overlap count, which follows from Lemma~\ref{lem:squaredSum}. We do not need to take the previous (higher) values into account. We can assume that these pairs do not overlap with the query vector since the previously added entries to $P$ would have already covered such a case. If we cannot reach the threshold, we stop the process for this centroid. Otherwise, we add the index with the determined count and the ID to the index map $P$ and proceed to the next pair\footnote{A naive implementation of this step would result in a worst-case time complexity of $\mathcal{O}(m^2)$ for adding a centroid with $m$ non-zero entries. We can perform this step in linear time, though. After reaching the threshold, we save the end position. Before we proceed to the next pair, we first subtract the squared value of the current pair from the sum. For the next pair, we can then continue the summation from the previous end position until we reach our threshold. Thus, we add and subtract each squared value at most once.}.
\end{enumerate}

Given an input vector $\mathbf{x}_i$ as query and the minimum dot product of $\lambda$, we determine the list of centroids for which we need to compute the cosine similarity as follows:

\begin{enumerate}
\item For each non-zero entry of $\mathbf{x}_i$, we retrieve the list of overlapping centroids using $G$ and increment our local count map $C$ for each of the IDs in the list. Given a centroid, we can then use $C$ to determine the number of overlapping non-zero entries with $\mathbf{x}_i$.
\item We iterate again through the entries of $\mathbf{x}_i$. For each entry, we retrieve the list of centroid candidates and the corresponding minimum overlap counts using $P$.
We add those candidates to our resulting set that meet the minimum overlap count, which we can determine using $C$.
\end{enumerate}

This indexing structure helps to accelerate the Spherical $k$-Means algorithm.
During the assignment step, we first build the index from the current centroids.
Given an input vector, we calculate the dot product with the centroid that corresponds to the assignment in the preceding iteration, which serves as a baseline.
If the calculated similarity is at least as high as our threshold $\lambda$, we can query our structure to retrieve a (possibly) shorter list of centroids for which we need to compute the dot product.
We can guarantee that all other centroids would result in a dot product $< \lambda$.

The higher $\lambda$, the smaller the expected number of items in our centroid list, but also the smaller the percentage of input items that meet the required baseline dot product.
To improve this trade-off, we build several such indexing structures with different values of $\lambda$.
Upon retrieval, we select the structure with the highest threshold which is still below the respective baseline value.
The general index map $G$ has to be built only once since it does not depend on the threshold.
It is possible to combine this strategy with the first one.
In this case, we would just further filter the returned list of centroids according to the rules outlined in the previous Section.

\section{Evaluation}

We tested our adapted document clustering approach on one million tweets and 200,000 ArXiv paper abstracts, with $k$ ranging between 50 and 5,000, to evaluate the impact of the strategies on the running time of the clustering.

\subsection{Test Setup}

We randomly sampled one million English tweets from a collection we fetched using the Twitter API and 200,000 paper abstracts (including the title) from ArXiv\footnote{https://www.kaggle.com/Cornell-University/arxiv}, excluding documents that only contain stop words to avoid zero vectors.
We tokenized the documents and converted them into a sparse TF-IDF-weighted Bag-of-Words representation, that is, for each present term in a document, we set the value of the corresponding dimension to the term frequency multiplied with the logarithm of the inverse document frequency of the term.
We ignored very frequent words (stop words) and divided each vector by its length to obtain unit-length vectors.
We did not truncate the vocabulary, leading to 325,556-dimensional paper abstract and 783,304-dimensional tweet vectors, each containing on average 58 and 10 non-zero entries, respectively.

We compared two modes with the baseline algorithm outlined in Section~\ref{sec:sphericalKMeans} using three different values of $k$: 50, 500, and 5,000.
The first mode only utilizes our non-changing clusters strategy (NCC), and the second one represents the full approach with both strategies enabled (NCC+INDEX).
For the full approach, we chose $\lbrace 0.1, 0.25, 0.4, 0.6 \rbrace$ as our set of minimum dot products.
For all modes, the clustering loop terminates if the assignments do not change anymore or the centroids largely stay the same (maximum squared Euclidean distance between any two subsequent centroids is $< 0.0001$).
We ran each configuration five times on a 32-core CPU and report on the median running time.

\subsection{Results}

\begin{table}
\centering
\caption{The median running times of our strategies (in minutes) on 1m tweets and 200k paper abstracts depending on different cluster sizes (lower is better). \emph{NCC} refers to our non-changing cluster strategy and \emph{INDEX} to our dot product indexing structure. \emph{NCC+INDEX} represents the full approach.}
\begingroup
\setlength{\tabcolsep}{4pt} 
\begin{tabular}{lrrrrrr}
\hline
	&	\multicolumn{3}{c}{\textbf{Tweets (1m)}} &	\multicolumn{3}{c}{\textbf{Abstracts (200k)}}  \\
	& $k = 50$ & $500$ & $5000$ & $k = 50$ & $500$ & $5000$ \\ \hline
Baseline              & 0.5 & 11.1 & 44.3 & 0.3 & 13.3 & 31.7 \\ \hline
NCC  	              & \textbf{0.4} &          6.1  & 18.1          & \textbf{0.3} & 7.3          & 13.9 \\
\textbf{NCC+INDEX}    & 0.7          &  \textbf{2.1} &  \textbf{2.0} & 0.6          & \textbf{5.6} & \textbf{5.5} \\ \hline

\end{tabular}
\label{tab:benchmark}
\endgroup
\end{table}

\begin{figure}
  \centering
  \includegraphics[width=\linewidth]{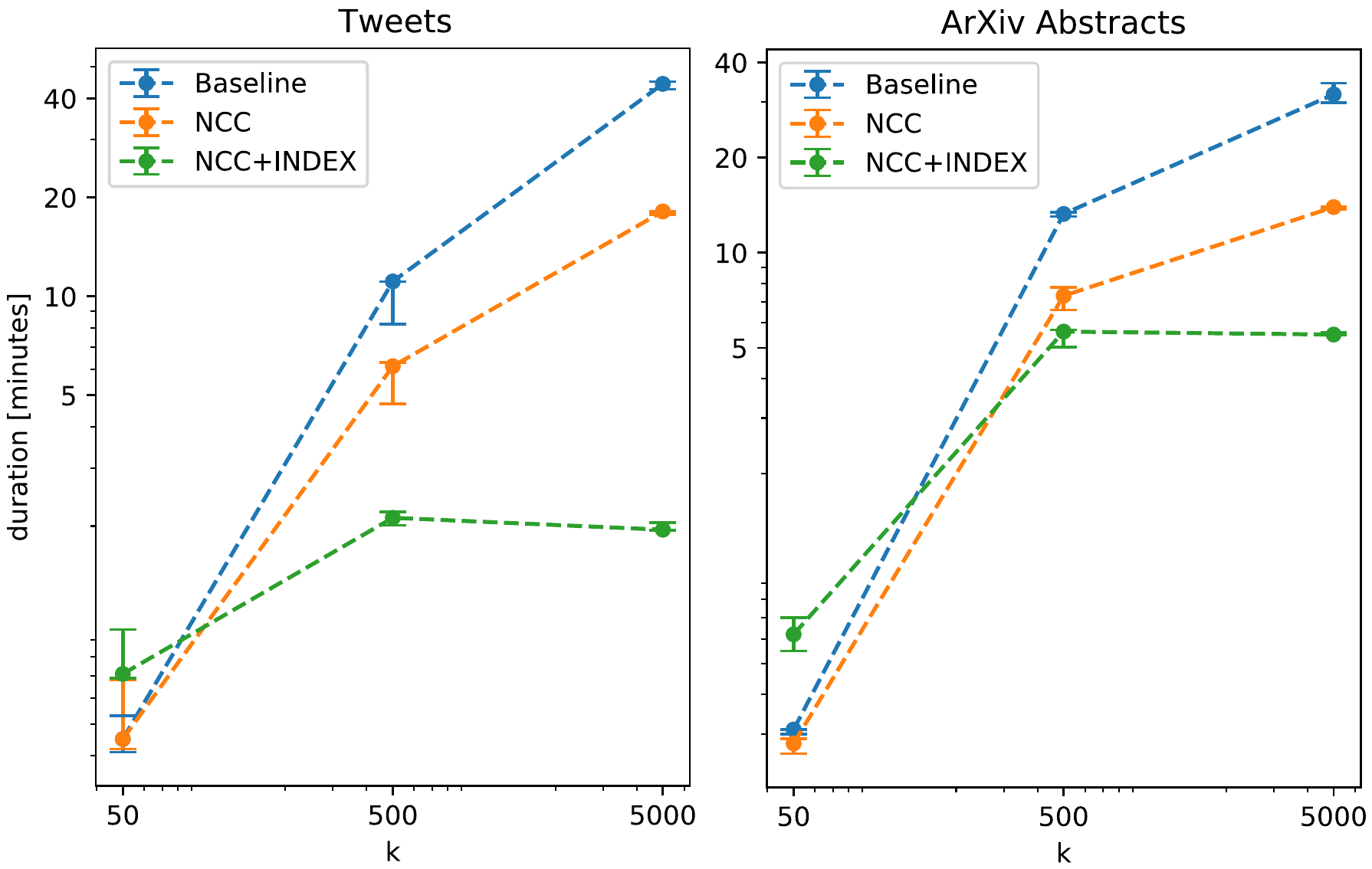}
  \caption{The median running time of our strategies, plotted on logarithmic scales. The bars indicate the interquartile range. The dots are connected with a dotted line to support the visual tracking of a specific configuration.}
  \label{fig:benchmark}
\end{figure}

Table~\ref{tab:benchmark} lists the results that are also plotted in Figure~\ref{fig:benchmark} on a logarithmic scale.
The error bars denote the interquartile ranges.
Our non-changing clusters strategy leads to shorter running times across all configurations in both data sets.
For larger values of $k$, the indexing structure further accelerates the clustering.
We observed a more than 20-fold reduction of the time it takes to cluster one million tweets into five thousand clusters, and a more than fivefold reduction in the case of the less sparse abstracts.
In contrast to the baseline scenario, clustering the documents into 5,000 instead of 500 clusters did not take more time with our approach on the two data sets.
For $k = 50$, our indexing structure cannot offset the additional overhead it introduces, resulting in slightly longer running times.

\subsection{Discussion}

In the case of sparse input vectors, our approach can significantly accelerate the Spherical $k$-Means algorithm.
Given a fixed number of input documents, the efficiency of the indexing structure typically increases with higher cluster sizes because the probability that a frequent word is an important component of many centroids decreases.
This explains why the running time does not seem to increase above $k = 500$.

Building the index on every iteration takes time, which does not pay off for smaller cluster sizes.
It is therefore advisable to enable the \emph{INDEX} strategy dynamically whenever the number of clusters that have changed compared to the previous iteration exceeds a certain threshold (e.g., 100).

It may seem odd that the duration it takes to process 500 clusters in the baseline scenario is more than two times longer than we would expect from the running time of processing 50 clusters.
One reason for this behavior is the available cache size.
If the number of clusters is sufficiently small, the centroids may fit completely into the cache of the processor, reducing expensive memory fetches.
For larger values of $k$ and, thus, increased memory usage, the percentage of cache misses increases, which reduces the computational efficiency significantly.

\section{Conclusion}

We proposed two strategies to accelerate the Spherical $k$-Means clustering algorithm for sparse input vectors such as document representations.
The first strategy exploits the observation that in later iterations, more and more cluster centroids remain stable between subsequent iterations.
The second strategy utilizes an indexing structure for unit-length vectors that we proposed.
Given an input vector, the index enables us to retrieve a filtered set of centroids for which we need to compute the cosine similarity with to find the most similar centroid.
Our benchmarks show that our approach leads to significant shorter running times, making a fine-grained cluster-based analysis of large document collections with $k \gg 10$ much more feasible.

\begin{acks}
This research was supported by the German Science Foundation (DFG) as part of the project VAOST (project number 392087235) and as part of the Priority Program VA4VGI (SPP 1894).
\end{acks}

\bibliographystyle{ACM-Reference-Format}
\bibliography{ess-kmeans-paper}

%
%
%
%
%
%
%
%

\end{document}